\documentclass{article} 
\usepackage{nips11submit_e,times}
\usepackage[utf8x]{inputenc}
\usepackage{times,dsfont}
\usepackage{amsmath}
\usepackage{amsfonts}
\usepackage{amssymb}
\usepackage{amsthm}
\usepackage{enumerate}
\usepackage{epsfig}

\newtheorem{lemma}{Lemma}
\newtheorem{theorem}{Theorem}
\newtheorem{corollary}{Corollary}
\newtheorem{definition}{Definition}

\newcommand{\X}{\mathcal{X}}
\newcommand{\Y}{\mathcal{Y}}
\newcommand{\Xspl}{\mathbf{X}}
\newcommand{\Yspl}{\mathbf{Y}}
\newcommand{\V}{\mathcal{V}}

\newcommand{\B}{\mathcal{B}}

\newcommand{\ft}{\widetilde{f}}

\newcommand{\real}{\mathbb{R}}
\newcommand{\expectation}{\mathbb{E}}

\newcommand{\expec}{{\Expectation} \,}
\newcommand{\expecf}[1]{{\Expectation_{#1}} \,}
\newcommand{\norm}[1]{\left\|#1\right\|}
\newcommand{\abs}[1]{\left|#1\right|}
\newcommand{\paren}[1]{\left(#1\right)}
\newcommand{\pr}[1]{\mathbb{P}\left(#1\right)}
\newcommand{\prf}[2]{\mathbb{P}_{#1}\left(#2\right)}

\newcommand{\distc}[1]{\rho\left(#1\right)}
\newcommand{\ind}[1]{\mathds{1}\left\{#1\right\}}
\newcommand{\braces}[1]{\left\{#1\right\}}

\DeclareMathOperator*{\volume}{vol}
\DeclareMathOperator*{\Expectation}{\expectation}

\newcommand{\vol}[1]{\volume\left(#1\right)}

\title{$k$-NN Regression Adapts to Local Intrinsic Dimension}
\author{Samory Kpotufe \\
Max Planck Institute for Intelligent Systems\\
\texttt{samory@tuebingen.mpg.de}}

\nipsfinalcopy
\begin{document}

\maketitle

\begin{abstract}
Many nonparametric regressors were recently shown to converge at rates that depend only on the intrinsic dimension of data.
These regressors thus escape the curse of dimension when high-dimensional data has low intrinsic dimension (e.g. a manifold). 
We show that $k$-NN regression is also adaptive to intrinsic dimension. In particular our rates are local to a query $x$ 
and depend only on the way masses of balls centered at $x$ vary with radius.

Furthermore, we show a simple way to choose $k = k(x)$ locally at any $x$ so as to nearly achieve the minimax rate at $x$
in terms of the unknown intrinsic dimension in the vicinity of $x$. We also establish that the minimax rate does not depend 
on a particular choice of metric space or distribution, but rather that this minimax rate holds for any metric space and doubling measure.
\end{abstract}

\section{Introduction}
We derive new rates of convergence in terms of dimension for the popular 
approach of Nearest Neighbor ($k$-NN) regression. 
Our motivation is that, for good performance, $k$-NN regression can require a number of samples exponential in the 
dimension of the input space $\X$. This is the so-called ``curse of dimension''.  Formally stated, 
the curse of dimension is the fact that, 
for any nonparametric regressor there exists a distribution in $\real^D$ such that, given a training size $n$, 
the regressor converges at a rate no better than $n^{-1/O(D)}$ (see e.g. \cite{S:60, S:61}).

Fortunately it often occurs that high-dimensional data has low intrinsic dimension: typical examples are 
data lying near low-dimensional manifolds \cite{manifold:lle, manifold:isomap, manifold:laplace_eigenmaps}. 
We would hope that in these cases nonparametric regressors can escape the curse of dimension, i.e. their performance 
should only depend on the intrinsic dimension of the data, appropriately formalized. 
In other words, if the data in $\real^D$ has intrinsic dimension $d<\!\!<D$, we would hope for a better 
convergence rate of the form $n^{-1/O(d)}$ instead of $n^{-1/O(D)}$.
This has recently been shown to indeed be the case for methods such as kernel regression \cite{BL:65}, tree-based regression 
\cite{SK:77} and variants of these methods \cite{K09}. In the case of $k$-NN regression however, it is only known that $1$-NN regression  
(where $k = 1$) converges at a rate that depends on intrinsic dimension \cite{KP:83}. Unfortunately $1$-NN regression is not consistent. 
For consistency, it is well known that we need $k$ to grow as a function of the sample size $n$ \cite{GKKW:81} .

Our contributions are the following. We assume throughout that the target function $f$ is Lipschitz. First we show that, for a wide range of values of $k$ ensuring consistency, 
$k$-NN regression converges at a rate that only depends on the intrinsic dimension in a neighborhood of a query $x$. 
Our local notion of dimension in a neighborhood of a point $x$ relies on the well-studied notion of \emph{doubling measure} (see Section \ref{sec:dim}). 
In particular our dimension quantifies how the mass of balls vary with radius, and this captures standard examples of data with low intrinsic dimension.
Our second, and perhaps most important contribution, is a simple procedure for choosing $k= k(x)$ so as to nearly achieve the minimax rate
of $O\paren{n^{-2/(2+d)}}$ in terms of the unknown dimension $d$ in a neighborhood of $x$. 
Our final contribution is in showing that this minimax rate holds for any metric space and doubling measure. In other words the hardness of the 
regression problem is not tied to a particular choice of metric space $\X$ or doubling measure $\mu$, 
but depends only on how the doubling measure $\mu$ expands on a metric space $\X$. Thus, for any marginal $\mu$ on $\X$
with expansion constant $\Theta\paren{2^d}$, the minimax rate for the measure space $(\X, \mu)$ is $\Omega\paren{n^{-2/(2+d)}}$.

\subsection{Discussion}
It is desirable to express regression rates in terms of a local notion of dimension 
rather than a global one because the complexity of data can vary considerably over regions of space. 
Consider for example a dataset made up of a collection of manifolds of various dimensions. The global complexity 
is necessarily of a worst case nature, i.e. is affected by the most complex regions of the space while we might happen to 
query $x$ from a less complex region. Worse, it can be the case that the data is not complex locally anywhere, but globally the data is more complex. 
A simple example of this is a so-called \emph{space filling curve} where a 
low-dimensional manifold curves enough that globally it seems to fill up space. We will see that the global complexity does not affect the behavior 
of $k$-NN regression, provided $k/n$ is sufficiently small. The behavior of $k$-NN regression 
is rather controlled by the often smaller local dimension in a neighborhood $B(x, r)$ of $x$, where the neighborhood size $r$ shrinks with $k/n$. 

Given such a neighborhood $B(x, r)$ of $x$, how does one choose $k=k(x)$ optimally relative to the unknown local 
dimension in $B(x, r)$?
This is nontrivial as standard methods of (global) parameter selection do not easily apply.
For instance, it is unclear how to choose $k$ by cross-validation over possible settings: we 
do not know reliable surrogates for the true errors at $x$ of the various estimators $\braces{f_{n, k}(x)}_{k\in [n]}$. 
Another possibility is to estimate the dimension of the data in the vicinity of $x$, and use this estimate to set $k$.
However, for optimal rates, we have to estimate the dimension exactly and 
we know of no finite sample result that guarantees the exact estimate of intrinsic dimension. 
Our method consists of finding a value of $k$ that balances quantities which control estimator variance and bias at $x$, 
namely $1/k$ and distances to $x$'s $k$ nearest neighbors.
The method guarantees, uniformly over all $x\in \X$, a near optimal rate of $\widetilde{O}
\paren{n^{-2/(2+d)}}$ where $d = d(x)$ is exactly the unknown local dimension on a neighborhood $B(x, r)$ of $x$, where $r\to 0$ as $n\to \infty$. 

 
\section{Setup}
We are given $n$ i.i.d samples $(\Xspl, \Yspl)=\braces{(X_i, Y_i)}_{i=1}^n$ from some unknown distribution
where the input variable $X$ belongs to 
a metric space $(\X, \rho)$, and the output $Y$ is a real number.  We assume that the class $\B$ of 
balls on $(\X, \rho)$ has finite VC dimension $\V_\B$. This is true for instance for any subset $\X$ of 
a Euclidean space, e.g. the low-dimensional spaces discussed in Section \ref{sec:dim}. The VC assumption is however irrelevant to the minimax result
of Theorem \ref{theo:optimalRates}.

We denote the marginal distribution on $\X$ by $\mu$ and the empirical distribution on $\Xspl$ by $\mu_n$. 

\subsection{Regression function and noise}
\label{sec:noise}
The regression function $f(x) = \expec [Y|X=x$] is assumed to be $\lambda$-Lipschitz, i.e. there exists $\lambda>0$ such that 
$\forall x, x' \in \X$, $\abs{f(x) - f(x')}\leq \lambda \distc{x, x'}$.

We assume a simple but general noise model: the distributions of the noise at points $x\in \X$ have uniformly bounded tails 
and variance. In particular, $Y$ is allowed to be unbounded. Formally:
$$\forall \delta>0 \text{ there exists } t>0 \text{ such that } \sup_{x\in \X} \prf{Y | X= x}{\abs{Y - f(x)}>t}\leq \delta.$$
We denote by $t_Y(\delta)$ the infimum over all such $t$. Also, we assume that the variance of $(Y|X=x)$ is upper-bounded
by a constant $\sigma^2_Y$ uniformly over all $x\in \X$. 

To illustrate our noise assumptions, consider for instance the standard assumption of bounded noise, i.e. 
$\abs{Y-f(x)}$ is uniformly bounded by some $M>0$; then $\forall \delta>0$, $t_Y(\delta)\leq M$, and can thus be replaced by $M$ in all 
our results.  Another standard assumption is that where the noise distribution has exponentially decreasing tail; in this case 
$\forall \delta>0$, $t_Y(\delta)\leq O(\ln 1/\delta)$. As a last example, in the case of Gaussian (or sub-Gaussian) noise, it's not hard to see that $\forall \delta>0$, $t_Y(\delta)\leq O(\sqrt{\ln 1/\delta})$.

\subsection{Weighted $k$-NN regression estimate}
We assume a kernel function $K:\real_+ \mapsto \real_+$, non-increasing, such that $K(1)>0$, and 
$K(\rho) = 0$ for $\rho>1$.
For $x \in \X$, let $r_{k, n}(x)$ denote the distance to its $k$'th nearest neighbor in the sample $\Xspl$. 
The regression estimate at $x$ given the $n$-sample $(\Xspl, \Yspl)$ is then defined as 
\begin{align*}
 f_{n, k}(x) = \sum_{i } \frac{K\paren{\rho(x, x_i)/r_{k, n}(x)}}{\sum_j K\paren{\rho(x, x_j)/r_{k, n}(x)}} Y_i = \sum_i w_{i, k}(x) Y_i. 
\end{align*}


\subsection{Notion of dimension}
\label{sec:dim}
We start with the following definition of doubling measure which will lead to the 
notion of local dimension used in this work. We stay informal in developing
the motivation and refer the reader to \cite{F116, C:90, C:74} for thorough overviews 
of the topic of metric space dimension and doubling measures. 
\begin{definition}
 The marginal $\mu$ is a {\bf doubling measure} if there exist $C_{\text{db}}>0$ such that 
for any $x\in \X$ and $r\geq 0$, we have $\mu(B(x, r)) \leq C_{\text{db}}\mu(B(x, r/2))$. The 
quantity $C_{\text{db}}$ is called an {\bf expansion constant} of $\mu$. 
\end{definition}
An equivalent definition is that, $\mu$ is doubling if there exist $C$ and $d$ such that for 
any $x\in \X$, for any $r\geq 0$ and any $0<\epsilon<1$,  
we have $\mu(B(x, r)) \leq C \epsilon^{-d}\mu(B(x, \epsilon r))$. Here $d$ acts as a dimension.
It is not hard to show that $d$ can be chosen as $\log_2 C_{\text{db}}$ and 
$C$ as $C_{\text{db}}$ (see e.g. \cite{LS115}). 

A simple example of a doubling measure is the Lebesgue volume in the Euclidean space $\real^d$.  
For any $x\in \real^d$ and $r>0$, $\vol{B(x,r)} = \vol{B(x, 1)} r^d$.
Thus $\vol{B(x,r)}/\vol{B(x, \epsilon r)} = \epsilon^{-d}$
for any $x\in \real^d$, $r>0$ and $0<\epsilon<1$. Building upon the doubling behavior of volumes in $\real^d$, we 
can construct various examples of doubling \emph {probability} measures. The following ingredients are sufficient. 
Let $\X\subset \real^D$ be a subset of a $d$-dimensional hyperplane, and let $\X$ satisfy 
for all balls $B(x, r)$ with $x\in \X$, $\vol{B(x, r)\cap \X} = \Theta(r^d)$, the volume being with respect to the containing hyperplane. 
Now let $\mu$ be approximately uniform, that is $\mu$ satisfies
for all such balls $B(x, r)$, $\mu(B(x, r)\cap \X) = \Theta(\vol{B(x, r)\cap \X})$. We then have $\mu(B(x,r))/\mu(B(x, \epsilon r)) = \Theta(\epsilon^{-d})$.


Unfortunately a global notion of dimension such as the above definition of $d$ is rather restrictive
as it requires the same complexity globally and locally. 
However a data space can be complex globally and have small complexity locally. 
Consider for instance a $d$-dimensional submanifold $\X$ of $\real^D$, and let $\mu$ have an upper and lower bounded
density on $\X$. The manifold might be globally complex but the restriction of $\mu$ to a ball $B(x, r), x\in \X$,  
is doubling with local dimension $d$, provided $r$ is sufficiently small and certain conditions on curvature hold.
This is because, under such conditions (see e.g. the \mbox{Bishop-Gromov} theorem \cite{manifold:do_carmo_riemann_geom}), 
the volume (in $\X$) of $B(x, r)\cap \X$  is $\Theta (r^d)$. 

The above example motivates the following definition of local dimension $d$. 

\begin{definition}
 Fix $x\in \X$, and $r>0$. 
Let $C\geq 1$ and $d\geq 1$.
The marginal $\mu$ is {\bf $(C, d)$-homogeneous on $B(x, r)$} if 
we have $\mu(B(x, r')) \leq C \epsilon^{-d}\mu(B(x, \epsilon r'))$ for all $r' \leq r$ and $0<\epsilon<1$.
\label{def:localdim}
\end{definition}

The above definition covers cases other than manifolds. In particular, another 
space with small local dimension is a sparse data space $\X\subset \real^D$ where each vector $x$ has at most $d$ non-zero coordinates, 
i.e. $\X$ is a collection of finitely many hyperplanes of dimension at most $d$. 
More generally suppose the data distribution $\mu$ is a mixture $\sum_i \pi_i \mu_i$
of finitely many distributions $\mu_i$ with potentially different low-dimensional supports. 
Then if all $\mu_i$ supported on a ball $B$ are $(C_i, d)$-homogeneous on $B$, i.e. all have local dimension $d$ on $B$, 
then $\mu$ is also $(C, d)$-homogeneous on $B$ for some $C$.

We want rates of convergence which hold uniformly over all regions where $\mu$ is doubling. 
We therefore also require (Definition \ref{def:uniform}) that $C$ and $d$ from Definition \ref{def:localdim} are uniformly upper bounded. 
This will be the case in many situations including the above examples.

\begin{definition}
\label{def:uniform}
The marginal $\mu$ is {\bf $(C_0, d_0)$-maximally-homogeneous} for some $C_0 \geq 1$ and $d_0\geq 1$, if the following holds for all $x\in\X$ and $r>0$: 
suppose there exists $C\geq 1$ and $d\geq 1$ such that $\mu$ is 
$(C, d)$-homogeneous on $B(x, r)$, then $\mu$ is 
$(C_0, d_0)$-homogeneous on $B(x, r)$. 
\end{definition}
We note that, rather than assuming as in Definition \ref{def:uniform} that all local dimensions are at most $d_0$, 
we can express our results in terms of the subset of $\X$ where local dimensions are at most $d_0$.
In this case $d_0$ would be allowed to grow with $n$. 
The less general assumption of Definition \ref{def:uniform} allows for a clearer presentation which still captures 
the local behavior of $k$-NN regression.


%
%

\section{Overview of results}

\subsection{Local rates for fixed $k$}
The first result below establishes the rates of convergence for any $k\gtrsim\ln n$ in terms of 
the (unknown) complexity on $B(x, r)$ where $r$ is any $r$ satisfying $\mu(B(x, r))> \Omega(k/n)$ 
(we need at least $\Omega(k)$ samples in the relevant neighborhoods of $x$).
\begin{theorem}
Suppose $\mu$ is $(C_0, d_0)$-maximally-homogeneous, and $\B$ has finite VC dimension $\V_\B$. Let $0<\delta<1$.
With probability at least $1-2\delta$ over the choice of $(\Xspl, \Yspl)$, 
the following holds simultaneously for all $x\in \X$ and $k$ satisfying $n >k \geq {\V_\B\ln {2n}  + \ln ({8}/{\delta})}$.

Pick any $x\in \X$. Let $r>0$ satisfy $\mu(B(x, r))>3C_0 k/n$.
Suppose $\mu$ is $(C, d)$-homogeneous on $B(x, r)$, with $1\leq C\leq C_0$ and $1\leq d\leq d_0$. We have

 \begin{align*}
\abs{f_{n, k}(x) - f(x)}^2 \leq\frac{2K(0)}{K(1)}\cdot \frac{\V_\B\cdot t_Y^2(\delta/2n)\cdot \ln(2n/\delta) + \sigma^2_Y}{  k }
+ 2\lambda^2  r^2 \paren{\frac{3Ck}{n  \mu(B(x, r))}}^{2/d}. 
 \end{align*}

\label{theo:fixedk}
\end{theorem}
Note that the above rates hold uniformly over $x$, $k\gtrsim\ln n$, and any $r$ where $\mu(B(x, r))\geq \Omega(k/n)$. 
The rate also depends on $\mu(B(x, r))$ and suggests that the best scenario is that where $x$ has a small neighborhood of large mass and small dimension $d$. 

\subsection{Minimax rates for a doubling measure}
Our next result shows that the hardness of the regression problem is not tied to a particular choice of the metric $\X$ or the 
doubling measure $\mu$. The result relies mainly on the fact that $\mu$ is doubling on $\X$. 
We however assume that $\mu$ has the same expansion constant everywhere and that this constant is tight. 
This does not however make the lower-bound less expressive, as it still tells us which rates to expect locally. 
Thus if $\mu$ is $(C, d)$-homogeneous near $x$, we cannot expect a better rate than $O\paren{n^{-2/(2+d)}}$ (assuming a Lipschitz regression function $f$).

\begin{theorem}
\label{theo:minimax}
Let $\mu$ be a doubling measure on a metric space $(\X, \rho)$ of diameter 1, and suppose $\mu$ satisfies, 
for all $x\in \X$, for all $r>0$ and $0<\epsilon<1$, $$C_1 \epsilon^{-d}\mu(B(x, \epsilon r))\leq \mu(B(x, r)) \leq C_2 \epsilon^{-d}\mu(B(x, \epsilon r)),$$
where $C_1$, $C_2$ and $d$ are positive constants independent of $x$, $r$, and $\epsilon$. Let $\Y$ be a subset of $\real$ and 
let $\lambda>0$. Define $\mathcal{D}_{\mu, \lambda}$ as the class of distributions on $\X \times \Y$ such that $X\sim\mu$ and 
the output $Y = f(X) + \mathcal{N}(0, 1)$ where $f$ is any $\lambda$-Lipschitz function from $\X$ to $\Y$. Fix a sample size $n>0$ and let 
$f_{n}$ denote any regressor on samples $(\Xspl, \Yspl)$ of size $n$, i.e. $f_{n}$ maps any such sample 
to a function $f_{n | (\Xspl, \Yspl)}(\cdot): \X\mapsto\Y$ in $L^2(\mu)$. 
There exists a constant $C$ independent of $n$ and $\lambda$ such that 
 \begin{align*}
  \inf_{\braces{f_{n}}}\sup_{\mathcal{D}_{\mu, \lambda}}\frac{\expec_{\Xspl, \Yspl, x}\abs{f_{n | (\Xspl, \Yspl)}(x) - f(x)}^2}{\lambda^{2d/(2+d)}n^{-2/(2+d)}}\geq C.
 \end{align*}
\end{theorem}

\subsection{Choosing $k$ for near-optimal rates at $x$}
Our last result shows a practical and simple way to choose $k$ locally so as to nearly achieve the minimax rate at $x$, i.e. 
a rate that depends on the unknown local dimension in a neighborhood $B(x, r)$ of $x$, where 
again, $r$ satisfies $\mu(B(x, r))> \Omega (k/n)$ for good choices of $k$. It turns out that we just need 
$\mu(B(x, r))> \Omega( n^{-1/3})$. 

As we will see, the choice of $k$ simply consists of monitoring the distances from $x$ to its nearest neighbors. 
The intuition, similar to that of a method for tree-pruning in \cite{SK:77}, is to look for a $k$ that 
balances the variance (roughly $1/k$) and the square bias (roughly $r_{k, n}^2(x)$) of the estimate. The procedure 
is as follows:
\begin{quote}
{\bf Choosing ${\bf  k}$ at ${\bf x}$:} Pick $\Delta \geq \max_{i}\distc{x, X_i}$, and pick $ \theta_{n, \delta}\geq \ln n/\delta$.\\
 Let $k_1$ be the highest integer in $[n]$ such that $\Delta^2\cdot \theta_{n, \delta}/k_1\geq r_{k_1, n}^2(x)$.\\
Define $k_2 = k_1 + 1$ and choose $k$ as $\arg \min_{k_i, i \in [2]} \paren{\theta_{n, \delta}/k_i + r_{k_i, n}^2(x)}$.
\end{quote}
The parameter $\theta_{n, \delta}$ \emph{guesses} how the noise in $Y$ affects the risk. This will soon be clearer. 
Performance guarantees for the above procedure are given in the following theorem. 
\begin{theorem}
Suppose $\mu$ is $(C_0, d_0)$-maximally-homogeneous, and $\B$ has finite VC dimension $\V_\B$. 
Assume $k$ is chosen for each $x\in\X$ using the above procedure, 
and let $f_{n, k}(x)$ be the corresponding estimate. 
Let $0<\delta<1$ and suppose  $n^{4/(6+ 3d_0)} > \paren{\V_\B\ln {2n}  + \ln ({8}/{\delta})}/\theta_{n, \delta}$.
With probability at least $1-2\delta$ over the choice of $(\Xspl, \Yspl)$, the following holds simultaneously for all $x\in \X$.

Pick any $x\in \X$. Let $0<r < \Delta$ satisfy $\mu(B(x, r))>6 C_0 n^{-1/3}$. 
Suppose $\mu$ is $(C, d)$-homogeneous on $B(x, r)$, with $1\leq C\leq C_0$ and $1\leq d\leq d_0$. We have
\begin{align*}
 \abs{f_{n, k}(x) - f(x)}^2 \leq\paren{\frac{2C_{n, \delta}}{\theta_{n, \delta}} + 2\lambda^2}\paren{1+4\Delta^2}\paren{\frac{3C\theta_{n, \delta}}{n  \mu (B(x, r))} }^{2/(2+d)},
\end{align*}
where $C_{n, \delta} = \paren{V_\B\cdot t_Y^2(\delta/2n)\cdot \ln(2n/\delta) + \sigma^2_Y} {K(0)}/{K(1)}$.
\label{theo:optimalRates}
\end{theorem}
Suppose we set $\theta_{n, \delta}=\ln^2 n/\delta$. Then, as per the discussion in Section \ref{sec:noise},
if the noise in $Y$ is Gaussian, we have $t_Y^2(\delta/2n)= O(\ln n/\delta)$, and therefore 
the factor $C_{n, \delta}/{\theta_{n, \delta}}=O(1)$. Thus ideally we want to set $\theta_{n, \delta}$ to the order of 
$(t_Y^2(\delta/2n)\cdot\ln n/\delta)$.

Just as in Theorem \ref{theo:fixedk}, the rates of Theorem \ref{theo:optimalRates} hold 
uniformly for all $x\in \X$, and all $0< r< \Delta$ where $\mu(B(x, r))>\Omega( n^{-1/3})$. 
For any such $r$, let us call $B(x, r)$ an \emph{admissible} neighborhood. 
It is clear that, as $n$ grows to infinity, w.h.p. any neighborhood $B(x, r)$ of $x$, $0< r < \sup_{x'\in \X} \distc{x, x'}$, becomes admissible.
Once a neighborhood $B(x, r)$ is admissible for some $n$, our procedure nearly attains the minimax rates in terms of the local dimension on $B(x, r)$, 
provided $\mu$ is doubling on $B(x, r)$. Again, the mass of an admissible neighborhood affects the rate, 
and the bound in Theorem \ref{theo:optimalRates} is best for an admissible neighborhood 
with large mass $\mu(B(x, r))$ and small dimension $d$.

\section{Analysis}
 Define $\ft_{n, k}(x) = \expecf{\Yspl | \Xspl} f_{n, k}(x) =  \sum_i w_{i, k}(x) f(X_i)$. 
We will bound the error of the estimate at a point $x$ in a standard way as 
\begin{align}
 \abs{f_{n, k}(x) - f(x)}^2 \leq  2\abs{f_{n, k}(x) - \ft_{n, k}(x)}^2 + 2 \abs{\ft_{n, k}(x) - f(x)}^2. \label{eq:biasvar}
\end{align}

Theorem \ref{theo:fixedk} is therefore obtained by combining bounds on the above two r.h.s terms (variance and bias). 
These terms are bounded separately in Lemma \ref{lem:bias} and Lemma \ref{lem:variance} below.

\subsection{Local rates for fixed $k$:  bias and variance  at $x$}
In this section we bound the bias and variance terms of equation (\ref{eq:biasvar}) with high probability, 
uniformly over $x\in \X$. We will need the following lemma which follows easily from standard VC theory \cite{VC:72}
results. The proof is given as supplement in the appendix.

\begin{lemma}
 \label{cor:VC}
Let $\B$ denote the class of balls on $\X$, with VC-dimension $\V_\B$. Let $0<\delta<1$, and define $\alpha_n =\paren{\V_\B\ln {2n}  + \ln ({8}/{\delta})}/{n}$.
The following holds with probability at least $1-\delta$ for all balls in $\B$. Pick any $a\geq \alpha_n$.  
Then  $\mu(B) \geq 3a \implies \mu_n(B)\geq a$ and $\mu_n(B) \geq 3a \implies \mu(B)\geq a$.
\end{lemma}

We start with the bias which is simpler to handle: it is easy to show that the bias of the estimate at $x$ depends on the 
radius $r_{k, n}(x)$. This radius can then be bounded, first in expectation using the doubling assumption on $\mu$, then
by calling on the above lemma to relate this expected bound to $r_{k, n}(x)$ with high probability. 

\begin{lemma}[Bias]
\label{lem:bias}
Suppose $\mu$ is $(C_0, d_0)$-maximally-homogeneous. Let $0<\delta<1$. 
With probability at least $1-\delta$ over the choice of $\Xspl$, the following holds 
 simultaneously for all $x\in \X$ and $k$ satisfying
$ n>k \geq \V_\B\ln {2n}  + \ln ({8}/{\delta})$.

Pick any $x\in \X$. Let $r>0$ satisfy $\mu(B(x, r))>3C_0 k/n$.
Suppose $\mu$ is $(C, d)$-homogeneous on $B(x, r)$, with $1\leq C\leq C_0$ and $1\leq d\leq d_0$.
We have:
 \begin{align*}
  \abs{\ft_{n, k}(x) - f(x)}^2 \leq \lambda^2  r^2 \paren{\frac{3C k}{n  \mu(B(x, r))}}^{2/d}. 
 \end{align*}

\end{lemma}
\begin{proof}

First fix $\Xspl$, $x\in \X$ and $k\in [n]$. We have:
\begin{align}
 \abs{\ft_{n, k}(x) - f(x)} &= \abs{ \sum_i w_{i, k}(x) \paren{f(X_i) - f(x)}} \leq  \sum_i w_{i, k}(x) \abs{f(X_i) - f(x)\nonumber}\\
&\leq \sum_i w_{i, k}(x) \lambda \distc{X_i, x} \leq \lambda r_{k, n}(x).\label{eq:bias}
\end{align}
We therefore just need to bound $r_{k, n}(x)$. We proceed as follows. 

Fix $x\in \X$ and $k$ and pick any $r>0$ such that $\mu(B(x, r))> 3C_0 k/n$. 
Suppose $\mu$ is $(C, d)$-homogeneous on $B(x, r)$, with $1\leq C\leq C_0$ and $1\leq d\leq d_0$.
Define 

$$\epsilon \doteq \paren{\frac{3Ck}{n \mu(B(x, r))}}^{1/d},$$ 

so that $\epsilon < 1$ by the bound on $\mu(B(x, r))$; then by the local doubling assumption on $B(x, r)$, 
we have $\mu(B(x, \epsilon r)) \geq C^{-1} \epsilon^d\mu(B(x,  r)) \geq 3 k/n$. Let $\alpha_n$ as defined in Lemma \ref{cor:VC}, 
and assume $k/n \geq \alpha_n$ (this is exactly the assumption on $k$ in the lemma statement). By Lemma \ref{cor:VC}, it follows that 
with probability at least $1-\delta$ uniform over $x$, $r$ and $k$ thus chosen, we have 
$\mu_n((B(x, \epsilon r))\geq k/n$ implying that $r_{k, n}(x) \leq \epsilon r$. 
We then conclude with the lemma statement by using equation (\ref{eq:bias}).
\end{proof}

\begin{lemma}[Variance]
Let $0<\delta<1$. With probability at least $1-2\delta$ over the choice of $(\Xspl, \Yspl)$, the following then holds  simultaneously for all $x\in \X$ and 
$k\in [n]$:

 \begin{align*}
\abs{f_{n, k}(x) - \ft_{n, k}(x)}^2 \leq\frac{K(0)}{K(1)}\cdot \frac{\V_\B\cdot t_Y^2(\delta/2n)\cdot \ln(2n/\delta) + \sigma^2_Y}{  k }.
 \end{align*}
\label{lem:variance}
\end{lemma}
\begin{proof}

First, condition on $\Xspl$ fixed. For any $x\in\X$, $k\in [k]$, let $\Yspl_{x, k}$ denote the subset of $\Yspl$ corresponding to points from 
$\Xspl$ falling in $B(x, r_{k, n}(x))$. For $\Xspl$ fixed, the number of such subsets $\Yspl_{x, k}$ is therefore at most 
the number of ways we can intersect balls in $\B$ with the sample $\Xspl$; this is in turn upper-bounded by $n^{\V_\B}$ as is well-known
in VC theory.

Let $\psi(\Yspl_{x, k})\doteq \abs{f_{n, k}(x) - \ft_{n, k}(x)}$. 
We'll proceed by showing that with high probability, for all $x\in \X$,  $\psi(\Yspl_{x, k})$ is close to its expectation, then we bound this expectation.

Let $\delta_0 \leq 1/2n$. We further condition on the event 
 $\Y_{\delta_0}$ that for all $n$ samples $Y_i$,  $\abs{Y_i-f(X_i)}\leq t_Y(\delta_0)$. By definition of $t_Y(\delta_0)$, 
the event $\Y_{\delta_0}$ happens with probability at least $1- n\delta_0 \geq 1/2$ . 
It follows that for any $x\in\X$ 
$$\expec\psi(\Yspl_{x, k})\geq\pr{\Y_{\delta_0}}\cdot \expecf{\Y_{\delta_0}}\psi(\Yspl_{x, k})\geq \frac{1}{2}\expecf{\Y_{\delta_0}}\psi(\Yspl_{x, k}) , $$
where $\expecf{\Y_{\delta_0}}[\cdot]$ denote conditional expectation under the event.
Let $\epsilon >0$, we in turn have 
\begin{align*}
 \pr{\exists x, k, \, \psi(\Yspl_{x, k}) > 2\expec \psi(\Yspl_{x, k}) + \epsilon} &\leq \pr{\exists x, k, \,\psi(\Yspl_{x, k}) > \expecf{\Y_{\delta_0}} \psi(\Yspl_{x, k}) + \epsilon}\\
&\leq \prf{\Y_{\delta_0}}{\exists x, k, \,\psi(\Yspl_{x, k}) > \expecf{\Y_{\delta_0}} \psi(\Yspl_{x, k}) + \epsilon} + n\delta_0.
\end{align*}
This last probability can be bounded by applying McDiarmid's inequality: changing any $Y_i$ value 
changes $\psi(\Yspl_{x, k})$ by at most $w_{i, k} \cdot t_Y(\delta_0)$ when we condition on the event $\Y_{\delta_0}$. This, followed by a union-bound
yields 
\begin{align*}
 \prf{\Y_{\delta_0}}{\exists x, k, \,\psi(\Yspl_{x, k}) > \expecf{\Y_{\delta_0}} \psi(\Yspl_{x, k}) + \epsilon} 
\leq n^{\V_\B}\exp\braces{-2\epsilon^2/t_Y^2(\delta_0)\sum_i w_{i, k}^2} .
\end{align*}

Combining with the above we get 
\begin{align*}
 \pr{\exists x\in \X, \, \psi(\Yspl_{x, k}) > 2\expec \psi(\Yspl_{x, k}) + \epsilon} 
\leq n^{\V_\B}\exp\braces{-2\epsilon^2/t_Y^2(\delta_0)\sum_i w_{i, k}^2} + n\delta_0.
\end{align*}
In other words, let  $\delta_0 = \delta/2n$, with probability at least $1-\delta$, for all $x\in\X$ and $k\in [n]$
\begin{align*}
 \abs{f_{n, k}(x) - \ft_{n, k}(x)}^2 &\leq 8\paren{\expecf{\Yspl|\Xspl}\abs{f_{n, k}(x) - \ft_{n, k}(x)}}^2 
+ {t_Y^2(\delta/2n)}\paren{\V_\B\ln(2n/\delta) \sum_i w_{i, k}^2}\\
&\leq 8\expecf{\Yspl|\Xspl}\abs{f_{n, k}(x) - \ft_{n, k}(x)}^2 + {t_Y^2(\delta/2n)}\paren{\V_\B\ln(2n/\delta) \sum_i w_{i, k}^2}, 
\end{align*}
where the second inequality is an application of Jensen's.

We bound the above expectation on the r.h.s. next. In what follows (second equality below) we use the fact 
that for i.i.d random variables $z_i$ with zero mean, $\expec \abs{\sum_i z_i}^2 = \sum_i \expec\abs{z_i}^2$. 
We have
\begin{align*}
\expecf{\Yspl| \Xspl}\abs{f_{n, k}(x) - \ft_{n, k}(x)}^2 &=  \expecf{\Yspl| \Xspl}\abs{\sum_i w_{i, k}(x) \paren{Y_i - f(X_i)}}^2\\
&= \sum_i w_{i, k}^2(x) \expecf{\Yspl| \Xspl}\abs{Y_i - f(X_i)}^2 \leq \sum_i w_{i, k}^2(x) \sigma^2_Y.
\end {align*}

Combining with the previous bound we get that, with probability at least $1-\delta$, for all $x$ and $k$,
\begin{align}
 \abs{f_{n, k}(x) - \ft_{n, k}(x)}^2 \leq \paren{\V_\B\cdot t_Y^2(\delta/2n)\cdot \ln(2n/\delta) + \sigma^2_Y} 
\cdot \sum_i w_{i, k}^2(x).\label{eq:var}
\end{align}
We can now bound $\sum_i w_{i, k}^2(x)$ as follows:
\begin{align*}
\sum_i w_{i, k}^2(x)&\leq \max_{i\in [n]} w_{i, k}(x) = \max_{i\in [n]}\frac{K\paren{\rho(x, x_i)/r_{k, n}(x)}}{\sum_j K\paren{\rho(x, x_j)/r_{k, n}(x)}}
\leq \frac{K(0)}{\sum_j K\paren{\rho(x, x_j)/r_{k, n}(x)}}\\
&\leq \frac{K(0)}{\sum_{x_j \in B(x, r_{k, n}(x))} K\paren{\rho(x, x_j)/r_{k, n}(x)}}
\leq \frac{K(0)}{K(1) k}.
\end{align*}
Plug this back into equation \ref{eq:var} and conclude.
\end{proof}

\subsection{Minimax rates for a doubling measure}
\label{sec:minimax}
The minimax rates of theorem \ref{theo:minimax} (proved in the appendix) are obtained as is commonly done by constructing a 
regression problem that reduces to the problem of binary classification (see e.g. \cite{S:60, S:61, GKKW:81}). 
Intuitively the problem of classification is hard in those instances where labels (say $-1, +1$) vary wildly over the 
space $\X$, i.e. close points can have different labels. We make the regression problem similarly hard. We will consider a 
class of candidate regression functions such that each function $f$ alternates between positive and negative in neighboring regions
($f$ is depicted as the dashed line below).
\begin{center}
 \includegraphics[width= 7cm]{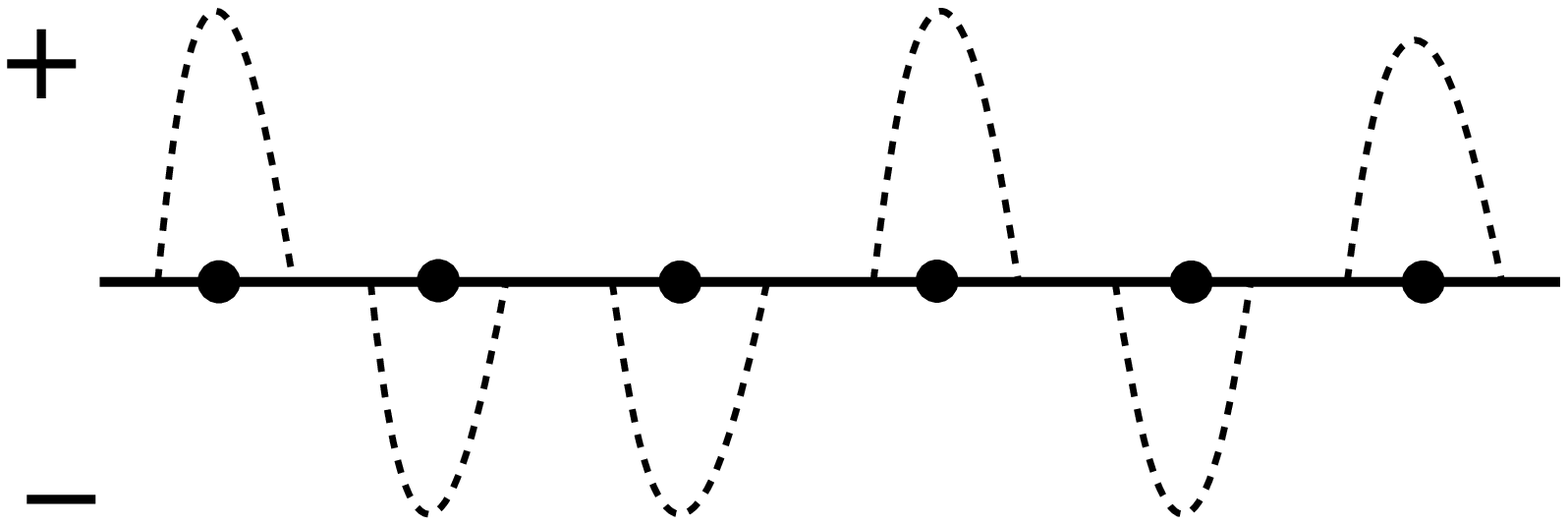}
\end{center}
The reduction relies on the simple observation that for a regressor $f_n$ to approximate the right $f$ from data it needs to 
at least identify the sign of $f$ in the various regions of space.
The more we can make each such $f$ change between positive and negative, the harder the problem. 
We are however constrained in how much $f$ changes since we also have to ensure that each $f$ is Lipchitz continuous. 

\subsection{Choosing $k$ for near-optimal rates at $x$}
\begin{proof}[Proof of Theorem \ref{theo:optimalRates}]
Fix $x$ and let $r, d, C$ as defined in the theorem statement. 
Define 
\begin{align*}
 \kappa \doteq \theta_{n, \delta}^{d/(2+d)}\cdot \paren{\frac{n  \mu (B(x, r))}{3C} }^{2/(2+d)}
\text{ and } \epsilon \doteq \paren{\frac{3C\kappa}{n \mu(B(x, r))}}^{1/d}.
\end{align*}
Note that, by our assumptions,
\begin{align}
 \mu(B(x, r))>6 C \theta_{n, \delta} n^{-1/3} \geq6 C\theta_{n, \delta} n^{-d /(2+d)} =6 C\theta_{n, \delta} \frac{n^{2/(2+d)}}{n} \geq 6 C \frac{\kappa}{n}. \label{eq:mass_at_r}
\end{align}
The above equation (\ref{eq:mass_at_r}) implies $\epsilon < 1$. Thus, by the homogeneity 
assumption on $B(x, r)$, 
$\mu(B(x, \epsilon r))\geq C^{-1} \epsilon^d\mu(B(x,  r))\geq 3{\kappa}/{n}$.
Now by the first inequality of (\ref{eq:mass_at_r}) we also have 
\begin{align*}
 \frac{\kappa}{n} \geq  \frac{\theta_{n, \delta}}{n} n^{4/(6+ 3d)} \geq \frac{\theta_{n, \delta}}{n} n^{4/(6+ 3d_0)}\geq \alpha_n,
\end{align*}
where $\alpha_n =\paren{\V_\B\ln {2n}  + \ln ({8}/{\delta})}/{n}$ is as defined in Lemma \ref{cor:VC}. We can thus apply Lemma \ref{cor:VC} to have that, 
with probability at least $1-\delta$, $\mu_n(B(x, \epsilon r))\geq \kappa/n$. In other words, for any $k\leq \kappa$, $r_{k, n}(x) \leq \epsilon r$. It follows that 
if $k\leq \kappa$, 
\begin{align*}
 \frac{\Delta^2\cdot \theta_{n, \delta}}{k} \geq \frac{\Delta^2\cdot \theta_{n, \delta}}{\kappa} = \Delta^2\paren{\frac{3C\kappa}{n \mu(B(x, r))}}^{2/d} \geq (\epsilon r)^2 \geq r_{k, n}^2(x).
\end{align*}
Remember that the above inequality is exactly the condition on the choice of $k_1$ in the theorem statement. 
Therefore, suppose $k_1\leq \kappa$, it must be that $k_2 > \kappa$ otherwise $k_2$ is the highest integer satisfying the condition, contradicting 
our choice of $k_1$. Thus we have (\rm{i})
$\theta_{n, \delta}/{k_2}< \theta_{n, \delta}/{\kappa} = \epsilon^2.$ We also have (\rm{ii}) $r_{k_2, n}(x) \leq 2^{1/d}\epsilon r$. To see this, notice that since 
$k_1\leq \kappa < k_2 = k_1 + 1$ we have $k_2 \leq 2\kappa$; now by repeating the sort of argument above, 
we have $\mu(B(x, 2^{1/d}\epsilon r))\geq 6{\kappa}/{n}$ which by Lemma \ref{cor:VC} implies that $\mu_n(B(x, 2^{1/d}\epsilon r))\geq 2{\kappa}/{n}\geq k_2/n$.

Now suppose instead that $k_1> \kappa$, then by definition of $k_1$, we have (\rm{iii})
\begin{align*}
 r_{k_1, n}(x)^2 \leq  \frac{\Delta^2\cdot \theta_{n, \delta}}{k_1} \leq \frac{\Delta^2\cdot \theta_{n, \delta}}{\kappa} = (\Delta \epsilon )^2 .
\end{align*}
The following holds by (\rm{i}), (\rm{ii}), and (\rm{iii}). Let $k$ be chosen as in the theorem statement. 
Then, whether $k_1> \kappa$ or not, it is true that 
\begin{align*}
 \paren{\frac{\theta_{n, \delta}}{k} + r_{k, n}^2(x)}\leq \paren{1+ 4\Delta^2}\epsilon^2= \paren{1+4\Delta^2} \paren{\frac{3C\theta_{n, \delta}}{n  \mu (B(x, r))} }^{2/(2+d)}.
\end{align*}
Now combine Lemma \ref{lem:variance} with equation (\ref{eq:bias}) and we have that with probability at least $1-2\delta$ (accounting for all events discussed) 
\begin{align*}
 \abs{f_{n, k}(x) -f(x)}^2 &\leq \frac{2C_{n, \delta}}{\theta_{n, \delta}} \frac{\theta_{n, \delta}}{k} + 2\lambda^2 r_{k, n}^2(x)
\leq \paren{\frac{2C_{n, \delta}}{\theta_{n, \delta}} + 2\lambda^2} \paren{\frac{\theta_{n, \delta}}{k} + r_{k, n}^2(x)}\\
&\leq\paren{\frac{2C_{n, \delta}}{\theta_{n, \delta}} + 2\lambda^2}\paren{1+4\Delta^2}\paren{\frac{3C\theta_{n, \delta}}{n  \mu (B(x, r))} }^{2/(2+d)}.
\end{align*}
\end{proof}

\section{Final remark}
The problem of choosing $k=k(x)$ optimally at $x$ is similar to the problem of local bandwidth 
selection for kernel-based methods (see e.g.  \cite{S:86, bandwidthSelect1}), and our method for 
choosing $k$ might yield insights into bandwidth selection, since $k$-NN and kernel regression methods only differ 
in their notion of neighborhood of a query $x$.

\subsubsection*{Acknowledgments}
I am grateful to David Balduzzi for many useful discussions.

\bibliographystyle{unsrt}
 \bibliography{refs}

\newpage
\appendix 
\section*{Appendix}
\section {Proof of Lemma \ref{cor:VC}}
\begin{lemma}[Relative VC bounds \cite{VC:72}]
Let $\B$ be a class of subsets of $\X$. Let $0< \delta < 1$. Suppose a sample $\Xspl$
of size $n$ is drawn independently at random from a distribution $\mu$ over $\X$
with resulting empirical distribution $\mu_n$. Define 
$\alpha_n =\paren{\V_\B\ln {2n}  + \ln ({8}/{\delta})}/{n}$.

Then with probability 
at least $1-\delta$ over the choice of $\Xspl$, all $B \in \B$ satisfy  
\begin{align*}
\mu(B) &\leq \mu_n(B) + \sqrt{ \mu_n(B)\alpha_n} + \alpha_n, \text{ and}\\
\mu_n(B) &\leq \mu(B) + \sqrt{\mu(B)\alpha_n} + \alpha_n.
\end{align*}
\label{lem:relativeVC}
\end{lemma}

Lemma \ref{cor:VC} is then obtained as the following corollary to Lemma  \ref{lem:relativeVC} above.

\begin{corollary}
 Let $\B$ denote the class of balls on $\X$. Let $0<\delta<1$, and as in Lemma \ref{lem:relativeVC}
above, define $\alpha_n =\paren{\V_\B\ln {2n}  + \ln ({8}/{\delta})}/{n}$.

The following holds with probability at least $1-\delta$ for all balls in $\B$. Pick any $a\geq \alpha_n$.  
\begin{itemize}
 \item $\mu(B) \geq 3a \implies \mu_n(B) + \sqrt{\mu_n(B)\alpha_n}\geq  \mu(B) - \alpha_n\geq a + \sqrt{a \alpha_n } \implies \mu_n(B)\geq a$.
\item $\mu_n(B) \geq 3a \implies \mu(B) + \sqrt{\mu(B)\alpha_n}\geq  \mu_n(B)- \alpha_n\geq a + \sqrt{ a\alpha_n } \implies \mu(B)\geq a$.
\end{itemize}
\end{corollary}

\section{Proof of Theorem \ref{theo:minimax}}
The minimax rates shown here are obtained as is commonly done by constructing a 
regression problem that reduces to the problem of binary classification (see e.g. \cite{S:60, S:61, GKKW:81}). 
Intuitively the problem of classification is hard in those instances where labels (say $-1, +1$) vary wildly over the 
space $\X$, i.e. close points can have different labels. We make the regression problem similarly hard. We will consider a 
class of candidate regression functions such that each function $f$ alternates between positive and negative in neighboring regions
($f$ is depicted as the dashed line below).
\begin{center}
 \includegraphics[width= 7cm]{LowerBound2.eps}
\end{center}
The reduction relies on the simple observation that for a regressor $f_n$ to approximate the right $f$ from data it needs to 
at least identify the sign of $f$ in the various regions of space.
The more we can make each such $f$ change between positive and negative, the harder the problem. 
We are however constrained in how much $f$ changes since we also have to ensure that each $f$ is Lipchitz continuous. 
Thus if $f$ is to be positive in some region and negative in another, these regions cannot be too close. We therefore have to break up the space $\X$ into a set of regions whose 
centers are far enough to ensure smoothness of the candidates $f$, yet in a way that the set is large so that each 
$f$ could be made to alternate a lot. We will therefore pick the centers of these regions as an $r$-net of $\X$
for some appropriate choice of $r$; by definition the centers would be $r$ far apart and together they would 
form an $r$-cover over space. As it turns out, any $r$-net is large under some tightness conditions on the expansion constants of $\mu$. 

We start with the following lemma which upper and lower bounds the size of an $r$-net 
under some tightness conditions on the doubling behavior of $\mu$. Results of this type appear in 
different forms in the literature. In particular similar upper-bounds on $r$-net size are discussed in 
\cite{C:74, LS115}. Here, we are mainly interested in the lower-bound on $r$-net size but show both upper and lower bounds for completion.
Both upper and lower bounds relie on union-bounds over sets of balls centered on a net.  
\begin{lemma}
Let $\mu$ be a measure on $\X$ such that 
for all $x\in \X$, for all $r>0$ and $0<\epsilon<1$, $$C_1 \epsilon^{-d}\mu(B(x, \epsilon r))\leq \mu(B(x, r)) \leq C_2 \epsilon^{-d}\mu(B(x, \epsilon r)),$$
where $C_1$, $C_2$ and $d$ are positive constants independent of $x$, $r$, and $\epsilon$.
Then, there exist $C_1', C_2'$ such that, for all $x\in \X$, for all $r>0$ and $0<\epsilon<1$,
  an $(\epsilon r)$-net of $B(x, r)$ has size at least 
$C_1' \epsilon^{-d}$ and at most $C_2' \epsilon^{-d}$. 
\label{lem:rnetsize}
\end{lemma}
\begin{proof}
Fix $B(x, r)$, and consider an $(\epsilon r)$-net $Z$ of $B(x, r)$. 
First we handle the upper-bound on $\abs{Z}$. 

Since, by a triangle inequality, we have for any $z\in Z$, $B(x, r) \subset B(z, 2r)$, 
it follows by the assumption on $\mu$ that, 
\begin{align*}
 \mu\paren{B\paren{z, \frac{\epsilon}{2} r}}\geq C_2^{-1} 4^{-d} \epsilon^d\mu\paren{B(z, 2r)}
\geq C_2^{-1} 4^{-d} \epsilon^d \mu\paren{B(x, r)}.
\end{align*}
Now, since for any such $z$, $B\paren{z, \frac{\epsilon}{2} r}\subset B(x, 2r)$, and the balls $B\paren{z, \frac{\epsilon}{2} r}$ for $z\in Z$ are disjoint (centers are $\epsilon r$ apart) we have 
\begin{align*}
 \abs{Z} C_2^{-1} 4^{-d} \epsilon^d \mu\paren{B(x, r)} \leq \mu\paren{ {\bigcup_{z\in Z}B\paren{z, \frac{\epsilon}{2} r}}}
\leq \mu(B(x, 2r)) \leq C_22^d \mu(B(x, r)), 
\end{align*}
implying that $\abs{Z}\leq C_2^2 8^d \epsilon^{-d}$.

The lower-bound on $\abs{Z}$ is handled similarly as follows. We have for all $z\in Z$
\begin{align*}
 \mu(B(z, \epsilon r)) \leq C_1^{-1} \epsilon^{d} \mu(B(z, r)) \leq C_1^{-1} \epsilon^{d} \mu(B(x, 2r)).
\end{align*}
Thus, applying a union-bound we obtain
\begin{align*}
 \abs{Z}C_1^{-1} \epsilon^{d} \mu(B(x, 2r)) \geq \mu\paren{ {\bigcup_{z\in Z}B\paren{z, \epsilon r}}}
\geq \mu(B(x, r)) \geq C_1 2^{-d} \mu(B(x, 2r)),
\end{align*}
implying that $\abs{Z}\geq C_1^{2} 2^{-d} \epsilon^{d}$.
\end{proof}

The proof of the minimax theorem follows.
\begin{proof}[Proof of Theorem \ref{theo:minimax}]
In what follows, for any function $g:\X \mapsto \real$, define $\norm{g}^2 = \expec_{x}\abs{g(x)}^2$, i.e. the squared norm of $g$ in $L^2(\mu)$.

Let $r_n = (\lambda^2 n)^{-1/(2+d)}$, and let $Z$ be an $r_n$-net of $\X$. Define $\tau=\min\braces{\frac{1}{3}C_1^{1/d}, \frac{1}{4}}$. 
For every $z\in Z$ we define the following function 
$$ f_z (x) \doteq \frac{\lambda}{5}\paren{\tau r_n -  \distc{x, z}}_+ \text{ i.e. } f_z(x) = 0 \text{ whenever } 
\distc{x, z}\geq \tau r_n.$$ It is easy to verify that for all $x, x'\in \X$
\begin{align*}
 \abs{f_z(x) - f_z(x')}\leq \frac{\lambda}{5}\abs{\distc{x, z} - \distc{x', z}} \leq \frac{\lambda}{5}\distc{x, x'}, 
\end{align*}
that is, $f_z$ is $\frac{\lambda}{5}$-Lipschitz. 

Now consider the random vector $\varsigma = \braces{\varsigma_z}_{z\in Z}$ where the coordinates 
$\varsigma_z\in \braces{-1, 1}$ are independent Bernoulli r.v.s that are $1$ with probability $1/2$. For 
every possible value of $\varsigma$ define 
$$f_\varsigma(x) \doteq \sum_{z\in Z}\varsigma_z g_z(x).$$
Next we verify that $f_\varsigma$ is $\lambda$-Lipschitz. First pick $x$ and $x'$ from the same ball 
$B(z, \tau r_n)$. It is clear that 
\begin{align*}
 \abs{f_\varsigma(x) - f_\varsigma(x')} = \abs{g_z(x) - g_z(x')}\leq \lambda\distc{x, x'}.
\end{align*}
Now suppose $x$ or $x'$ or both are outside all balls $\braces{B(z, \tau r_n)}_{z\in Z}$, then again it is easy to see that 
$\abs{f_\varsigma(x) - f_\varsigma(x')} \leq \lambda\distc{x, x'}.$ We now check the final case where $x\in B(z, \tau r_n)$ and 
$x' \in B(z', \tau r_n)$, $z\neq z'$. To this end, first notice that the ring $B(z, r_n/2)\setminus B(z, \tau r_n)$ is non-empty since 
$\mu(B(z, r_n/2))\geq C_1(2\tau)^{-d}\mu(B(z, \tau r_n)) > \mu(B(z, \tau r_n))$. 
Pick $x''$ in this ring, and notice that $x''$ is outside both balls $B(z, \tau r_n)$ and $B(z', \tau r_n)$. Thus, 
$g_z(x'') = g_z'(x'') = 0$, and we can write
\begin{align*}
 \abs{f_\varsigma(x) - f_\varsigma(x')} &= \abs{g_z(x) - g_z'(x')} \leq  \abs{g_z(x) - g_z(x'')} + \abs{g_z'(x'') - g_z'(x')}\\
&\leq \frac{\lambda}{5}\paren{\distc{x, x''} + \distc{x'', x'}}\leq \frac{\lambda}{5}\paren{2\distc{x,x''} + \distc{x, x'}}\\
&\leq \frac{\lambda}{5}\paren{2r_n + \distc{x, x'}}\leq \frac{\lambda}{5}\paren{4\distc{x, x'} + \distc{x, x'}}= \lambda\distc{x, x'}.
\end{align*}
At this point we can define $\mathcal{D}$ as the class of distributions on $\X\times \Y$, where $X\sim \mu$ and 
$Y = f_\varsigma(X) + \mathcal{N}(0, 1)$, for some $f_\varsigma$ as constructed above. Clearly 
$\mathcal{D}\subset \mathcal{D}_{\mu, \lambda}$ and we just have to show that
$$\inf_{\braces{f_n}}\sup_{\mathcal{D}}\frac{\expec\norm{f_n - f_\varsigma}^2}{\lambda^2 r_n^2}\geq O(1).$$

Fix a regressor $f_n$, that is $f_n$ maps any sample $(\Xspl, \Yspl)$ to a function in $L^2(\mu)$, which, 
for simplicity of notation, we also denote by $f_n$. For $(\Xspl, \Yspl)$ fixed, we denote by 
$f_{n, Z}$ the projection (in the Hilbert space $L^2(\mu)$) of $f_n$ onto the orthonormal system 
$\braces{f_z/\norm{f_z}}_{z\in Z}$. In other words, 
$f_{n, Z} = \sum_{z\in Z}\frac{<f_n, f_z>}{\norm{f_z}^2}f_z= \sum_{z\in Z}w_{n, z} f_z.$ Thus, for $f_\varsigma$ fixed, we
have 
\begin{align*}
 \expecf{\Xspl, \Yspl}\norm{f_n - f_\varsigma}^2 \geq  \expecf{\Xspl, \Yspl}\norm{f_{n, Z} - f_\varsigma}^2
= \expecf{\Xspl, \Yspl}\sum_{z\in Z}\paren{w_{n, Z}-\varsigma_z}^2\norm{f_z}^2.
\end{align*}
To bound $\norm{f_z}$, notice that $f_z$ is at least $\lambda\tau r_n/10$ on the ball $B(z, \tau r_n/2)$. This ball 
in turn has mass at least $C_2^{-1}(\tau r_n/2)^{d}$, so $\norm{f_z}\geq C_3 /\sqrt{n}$, for $C_3$ appropriately chosen.
We therefore have 
\begin{align*}
\sup_{\mathcal{D}} \expecf{\Xspl, \Yspl}\norm{f_n - f_\varsigma}^2 
&\geq \frac{C_3^2}{n} \sup_{\mathcal{D}} \expecf{\Xspl, \Yspl}\sum_{z\in Z}\paren{w_{n, z}-\varsigma_z}^2
\geq \frac{C_3^2}{n} \sup_{\mathcal{D}} \expecf{\Xspl, \Yspl}\sum_{z\in Z}\ind{w_{n, z}\cdot \varsigma_z < 0}\\
&\geq \frac{C_3^2}{n} \,\displaystyle\expecf{\varsigma}\expecf{\Xspl, \Yspl}\sum_{z\in Z}\ind{w_{n, z}\cdot \varsigma_z < 0}
\geq \frac{C_3^2}{n} \expecf{\Xspl}\sum_{z\in Z}\expecf{Y, \varsigma|\Xspl}\ind{w_{n, z}\cdot \varsigma_z < 0}.
\end{align*}
For $z\in Z$ fixed, $\expecf{Y, \varsigma|\Xspl}\ind{w_{n, Z}\cdot \varsigma_z < 0}$ is the probability of error 
of a classifier (which outputs $\text{sign}(w_{n, z})$) for the following prediction task. Let $x_{(1)}, x_{(2)}, \ldots x_{(m)}$ 
denote the values of $X$ falling in $B(z, \tau r_n)$ where $f_z$ is non zero. Then 
$$ (Y_{(1)}, \ldots Y_{(m)}) = \varsigma_z(f_z(x_{(1)}), \ldots, f_z(x_{(m)})) + \mathcal{N}(0, I_m)$$ 
is a random vector sampled from the equal-weight mixture of two spherical Gaussians in $\real^m$ centered at 
$u \doteq (f_z(x_{(1)}), \ldots, f_z(x_{(m)}))$ and $-u$. The prediction task is that of identifying the right mixture component
from the single sample $(Y_{(1)}, \ldots Y_{(m)})$. The smallest possible error for this task is that of the Bayes classifier and is well 
known to be $\Phi(-\norm{u})\geq \Phi\paren{-\sqrt{\sum_{i=1}^n f_z^2(X_i)}}$. Since 
$\Phi(-\sqrt{\cdot})$ is convex, we can apply Jensen's inequality as follows.
\begin{align*}
\sup_{\mathcal{D}} \expecf{\Xspl, \Yspl}\norm{f_n - f_\varsigma}^2 
&\geq \frac{C_3^2 }{n}\sum_{z\in Z}\expecf{\Xspl}\Phi\paren{-\sqrt{\sum_{i=1}^n f_z^2(X_i)}}
\geq \frac{C_3^2 }{n}\sum_{z\in Z}\Phi\paren{-\sqrt{\sum_{i=1}^n \expecf{\Xspl} f_z^2(X_i)}}\\
&= \frac{C_3^2}{n}\sum_{z\in Z}\Phi\paren{-\sqrt{n \norm{f_z}^2}}.
\end{align*}
The norm $\norm{f_z}$ is at most $C_3'/\sqrt{n}$ since $f_z\leq\lambda\tau r_n$ everywhere, and 
non zero only on the ball $B(z, \tau r_n)$ which has mass at most $C_1^{-1}(\tau r_n)^{d}$.
Finally, remember that by Lemma \ref{lem:rnetsize}, $\abs{Z}\geq C_1' r_n^{-d}$. Using these two facts we have
\begin{align*}
\sup_{\mathcal{D}} \frac{1}{\lambda^2 r_n^2}\expecf{\Xspl, \Yspl}\norm{f_n - f_\varsigma}^2
\geq \frac{C_3^2 }{ \lambda^2 r_n^2 n}\sum_{z\in Z}\Phi(-C_3')
\geq C_3^2 C_1'\tau^d\Phi(-C_3'),
\end{align*}
which concludes the proof since $f_n$ is arbitrarily chosen.
\end{proof}

\end{document}